%% file: main.tex
\documentclass{article}

\usepackage{iclr2027_conference,times}
\input{math_commands.tex}

\usepackage{amsmath,amssymb,amsfonts,amsthm,mathtools}
\usepackage{bm}
\usepackage{booktabs}
\usepackage{graphicx}
\graphicspath{{./}}
\usepackage{multirow}
\usepackage{array}
\usepackage[table]{xcolor}
\usepackage{url}
\usepackage{hyperref}
\usepackage{cleveref}
\usepackage{algorithm}
\usepackage{algpseudocode}
\usepackage{subcaption}
\usepackage{microtype}

\newtheorem{theorem}{Theorem}
\newtheorem{proposition}{Proposition}

\DeclareRobustCommand{\method}{\mbox{\textsc{ART-MPG}}}
\newcommand{\methodmath}{\mathrm{ART\text{-}MPG}}
\newcommand{\ind}{\mathbf{1}}

\title{\method: Autoregressive R/T Manifold Path Guiding}

\author{
Zehao Zhao\textsuperscript{1},
Yanyang Yang\textsuperscript{2},
Xiaole Zhang\textsuperscript{1},
Xiyu Deng\textsuperscript{3,}\thanks{Corresponding author: \texttt{xiyud@andrew.cmu.edu}.} \\
{\small
\textsuperscript{1}Ming Hsieh Department of Electrical and Computer Engineering, University of Southern California}\\
{\small
\textsuperscript{2}Computational Science and Engineering, Georgia Institute of Technology}\\
{\small
\textsuperscript{3}Department of Electrical and Computer Engineering, Carnegie Mellon University}
}

\iclrfinalcopy

\begin{document}

\maketitle
\lhead{Preprint}

\begin{abstract}
\input{sections/abstract}
\end{abstract}

\input{sections/introduction}
\input{sections/related_work}
\input{sections/background}
\input{sections/problem_formulation}
\input{sections/algorithm_design}

\input{sections/experiments}

\input{sections/limitations}
\input{sections/conclusion}
\input{sections/reproducibility}

\bibliography{references}
\bibliographystyle{iclr2027_conference}

\appendix
\input{sections/appendix}

\end{document}

%% file: math_commands.tex
\usepackage{amsmath,amsfonts,bm}

\def\eqref#1{equation~\ref{#1}}

\def\1{\bm{1}}

\DeclareMathAlphabet{\mathsfit}{\encodingdefault}{\sfdefault}{m}{sl}
\SetMathAlphabet{\mathsfit}{bold}{\encodingdefault}{\sfdefault}{bx}{n}



%% file: sections/abstract.tex
Specular chains concentrate important transport in small regions of path space and the structure depends on chain topology.
We introduce \method{}, a compact neural model that samples chain length, an autoregressive
reflection/transmission (R/T) string, and a direction from a topology-conditioned four-component
von Mises--Fisher (vMF) mixture.  
We extend the repeated trial weighting in Manifold Path Guiding (MPG) by conditioning retry sampling on both chain length and the sampled R/T string, accounting for its known mass outside the stochastic re-hit trials. For a fixed
proposal and target chain, this preserves the ideal expectation while reducing the expected retry count and reciprocal-estimator
variance compared with retries conditioned on chain length alone.
We distinguish the formal unbiasedness result from the capped and clamped renderer used in experiments.
We evaluate fourteen scenes across a range of samples per pixel (SPP) against existing online and offline rendering algorithms. Our method consistently produces lower mean squared error at equal SPP than the online and offline baselines, with lower standard deviation across runs.

%% file: sections/introduction.tex
\section{Introduction}
\label{sec:introduction}

Specular transport produces visually important paths that occupy narrow,
structured regions of path space. Ordinary path tracing can miss these paths,
while samples that find them may carry substantial energy.  Manifold methods
address the geometric difficulty by constructing an approximate seed chain and
projecting it onto the manifold defined by reflection and refraction constraints
\citep{jakob2012manifold,hanika2015mnee,zeltner2020specular}.

Manifold Path Guiding (MPG) learns from previously found admissible chains. It
factorizes the seed proposal into a distribution over chain lengths, a conditional distribution over complete reflection/transmission (R/T) strings, and
a type-conditioned directional von Mises--Fisher (vMF) mixture, reconstructing these distributions
online from local path records~\citep{fan2023manifold_path_guiding}. This is a
powerful representation, but its complete-type tables grow with the number of
R/T strings, its directional mixture grows with local records, and its
repeated-trial estimator factors out only the known length probability.

While diffusion has long been used to generate 3D space representations~\citep{zhang2026hilbertgen3dhilbertmultifractalconditioningtopologyaware3d}, its application in the light rendering problem is still understudied. We introduce \method{} (autoregressive R/T manifold path guiding), a compact
neural amortization of this MPG pipeline. A prefix model assigns an exactly evaluable probability mass
to each ordered R/T string, and a fixed four-component vMF head predicts the
direction density conditioned on the completed descriptor. The exact type
mass can then be divided out before re-hit estimation. Under a common mixed
proposal and a fixed target chain, conditioning on both length and type, rather than
length alone, preserves the ideal expectation while reducing the variance of the reciprocal estimator. Under the same finite retry cap, it
also cannot increase the timeout probability. 
Our theoretical contribution is the comparison between conditioning retries on chain length alone and conditioning them on both chain length and the R/T string, rather than the established mixed factorization, vMF family, or geometric re-hit identity.

The many-to-one nature of manifold projection remains central: a pointwise directional
density is not the probability of producing a projected chain. We therefore
separate three objects throughout the paper: the normalized neural seed
proposal, the ideal descriptor-conditioned chain estimator inherited from the
Specular Manifold Sampling (SMS)/MPG re-hit lineage, and the finite-work capped/clamped renderer actually
evaluated. This distinction makes explicit how the theoretical guarantees relate to the evaluated implementation.

Our contributions are:
\begin{enumerate}
    \item We introduce a mathematical formulation for modeling
  reflection/transmission sequences in specular transport, coupled
  with a directional proposal conditioned on the sampled chain
  length and complete R/T sequence (Section~\ref{sec:mpgd-proposal}, ~\ref{sec:mpgd-training})
    \item We develop a descriptor-conditioned re-hit estimator that
  factors out the exactly evaluable length-and-type probability.
  For a common proposal and a fixed target chain, we prove that
  this construction preserves the ideal expectation and reduces
  reciprocal-estimator variance whenever the target type probability
  is less than one, with equality otherwise
  (Section~\ref{sec:mpgd-chain-space}, Equation~\ref{eq:mpgd-unbiased-estimator}, and Proposition~\ref{prop:type-stratification}).
\end{enumerate}

%% file: sections/related_work.tex
\section{Related Work}
\label{sec:related-work}

\paragraph{Path guiding and learned proposals.}
Classical path guiding fits spatially varying directional distributions from
transport samples, including parametric mixtures, product-guided densities,
and parallax- or variance-aware models~\citep{vorba2014online,
herholz2016product,muller2017practical,ruppert2020robust,rath2020variance}.
Neural approaches instead amortize sampleable densities with compact networks
or invertible mappings~\citep{muller2019neural,zheng2019primary,bako2019offline,
dong2023npm}.  These works motivate our normalized neural proposal, but ordinary
directional guiding does not by itself account for discrete specular topology
or the many-to-one manifold projection.

\paragraph{Specular transport.}
Manifold exploration, MNEE, and SMS construct paths by solving reflection and
refraction constraints~\citep{jakob2012manifold,hanika2015mnee,
zeltner2020specular}.  SMS also supplies the repeated-reconstruction lineage
behind geometric reciprocal estimation.  Deterministic root-finding methods,
including Specular Polynomials, offer a complementary route to discovering
admissible chains~\citep{fan2024specular}; our work addresses proposal learning
for the manifold-walk pipeline rather than replacing its solver.

\paragraph{Relation to MPG.}
MPG already decomposes a seed law into length, complete R/T type, and a
type-conditioned directional vMF mixture.  It reconstructs these factors
online from local path records and, in its Eq.~15, holds length fixed during
SMS-style re-hit trials~\citep{fan2023manifold_path_guiding}.  \method{} does
not claim the mixed factorization, vMF family, or generic re-hit identity as
new.  It neurally amortizes MPG with a prefix-autoregressive type model and a
fixed-size four-component vMF head, then extends the exact discrete factor in the retry
estimator from length to length plus ordered type.  The latter change removes
the type probability from stochastic re-hit trials, as quantified in
Proposition~\ref{prop:type-stratification}.

%% file: sections/background.tex
\section{Background and Preliminaries}
\label{sec:background}

\paragraph{Monte Carlo light transport.}
For a surface point $y\in\mathcal M\subset\mathbb R^3$ and a unit outgoing direction $\omega_o\in\mathbb S^2$, the rendering
equation writes outgoing radiance as~\citep{kajiya1986rendering}
\begin{equation*}
  L_o(y,\omega_o)=L_e(y,\omega_o)+
  \int_{\mathcal H_y^2}
  f_s(y,\omega_i,\omega_o)L_i(y,\omega_i)
  |\cos\theta_i|\,\mathrm d\omega_i .
\end{equation*}
Here, $\mathcal M$ denotes the scene surfaces, $\mathbb S^2=\{\omega\in\mathbb R^3:\|\omega\|_2=1\}$ is the
unit sphere, and $\mathcal H_y^2$ is the hemisphere above the surface at $y$. The terms $L_e$, $L_i$, and $f_s$ denote emitted radiance, incident
radiance, and the surface scattering function, respectively;
$\theta_i$ is the angle between the incident direction $\omega_i$
and the surface normal.
If $\omega_i\sim p(\cdot\mid y,\omega_o)$, a one-sample estimator of the
integral is the integrand divided by $p(\omega_i\mid y,\omega_o)$, where $p$ is a density with respect to solid angle and is positive
wherever the integrand is nonzero.
Consequently, a proposal that places probability near high-contribution
directions can reduce variance; one that undersamples a narrow contribution
region produces rare, high-weight samples.  Path guiding learns such
proposals from transport observations~\citep{vorba2014online,
muller2017practical,ruppert2020robust}.  In our setting, a seed direction initializes a nonlinear manifold solve, and multiple seeds can converge to the same final path. The probability of reconstructing that path therefore depends on the total probability mass of those seeds, rather than the pointwise density of the particular seed that found it.

\paragraph{Specular chains and separator vertices.}
Following manifold path sampling and MPG~\citep{jakob2012manifold,
fan2023manifold_path_guiding}, we decompose a light path at non-specular
\emph{separator} vertices.  Between two consecutive separators, a specular chain contains
ideal reflection or refraction vertices jointly constrained by geometry and
Snell's law.  
We describe its discrete configuration and initial seed direction by
\begin{equation*}
  d=(n,t_{1:n}),\qquad t_i\in\{\mathrm R,\mathrm T\},\qquad
  \omega\in\mathbb S^2
\end{equation*}
where $n$ is the number of specular vertices, $t_{1:n}$ is their ordered
R/T strings, and $\omega$ is the initial seed
direction. Sampling $n$, $t_{1:n}$, and $\omega$ constructs an initial path;
a manifold walk then adjusts its vertices to satisfy the coupled constraints
or reports failure~\citep{hanika2015mnee,zeltner2020specular}.  The resulting
chain contributes the usual product of endpoint factors, scattering terms,
geometry terms, and transmittance.  The challenge is therefore not merely to
predict a bright direction, but to allocate probability across disconnected
length-and-type strata before predicting a direction within the selected
stratum.

\paragraph{Projection basins and chain probability.}
For fixed outer renderer state $z$, let $\Phi_z(d,\omega)$ denote the
deterministic manifold solve.  The map from seeds to admissible chains is
many-to-one: for a given descriptor $d$, seed directions within an angular \emph{convergence basin} can project to the same
canonical chain.  If $q(d,\omega)$ is the mixed seed proposal, the probability of a chain is the integral of $q$ over its preimage under $\Phi_z$, not the pointwise density of the particular seed that found it. SMS estimates the reciprocal basin mass by repeated reconstruction, while MPG factors known discrete probability mass from this trial process~\citep{zeltner2020specular, fan2023manifold_path_guiding}. Section~\ref{sec:problem-formulation} formalizes this chain-space probability; Section~\ref{sec:mpgd-chain-space} describes the length-and-type stratification used by \method{}.

\paragraph{Directional proposals on the sphere.}
For a unit mean direction $\mu\in\mathbb S^2$ and concentration $\kappa\ge0$,
the vMF density on $\mathbb S^2$ is
\begin{equation}
  \operatorname{vMF}(\omega;\mu,\kappa)
  =\frac{\kappa}{4\pi\sinh\kappa}
   \exp(\kappa\mu^\top\omega),
  \label{eq:bg-vmf}
\end{equation}
with the $\kappa\rightarrow0$ limit equal to the uniform spherical density $1/(4\pi)$.
A mixture $\sum_j\rho_j\operatorname{vMF}(\omega;\mu_j,\kappa_j)$ represents multiple directional modes while retaining a normalized, 
analytically evaluable, and directly sampleable density. Neural path guiding, NPM, and MPG use vMF mixtures for spherical transport proposals~\citep{ruppert2020robust,
dong2023npm,fan2023manifold_path_guiding}.  \method{} predicts a four-component mixture conditioned on the complete sampled R/T descriptor.  Numerical conventions and exact sampling details are provided in Appendix~\ref{app:vmf-details}.

\paragraph{Discrete diffusion for R/T sequences.}
For a chain of length $n$, the reflection/transmission sequence
$t_{1:n}\in\{\mathrm R,\mathrm T\}^n$ specifies the interaction type
at each specular vertex. Discrete diffusion models such sequences
through a forward corruption process and a learned reverse process~\citep{xu2026polynomial, xu2024provably, banerjee2025adaptive, hoshinotraining}.
The forward process progressively corrupts the categorical tokens,
for example by masking or replacing them. Conditioned on the
corrupted sequence, the diffusion step, and the transport context,
a denoising network predicts probabilities for the R/T tokens.
Generation proceeds through successive reverse steps to obtain
a complete R/T sequence, which can then condition the directional
proposal used to initialize a specular chain.

%% file: sections/problem_formulation.tex
\section{Problem Formulation}
\label{sec:problem-formulation}

Let $z$ denote the renderer state before seed sampling, including the
separator vertices, the sampled emitter point, and the outcomes of preceding
random choices. The network takes the corresponding features $x=\psi(z)$
as input. Our analysis assumes that $z$ remains fixed during the initial
solve and all subsequent trials. Seed-dependent sampling of glossy offsets falls outside these assumptions, as discussed in \Cref{app:ideal-estimator}. 
Let $\mathcal A(z)$ denote the set of admissible specular chains with at
most $N_{\max}$ vertices, and let $F_z(a)$ be the RGB contribution of
chain $a$ under state $z$. We seek to estimate
\begin{equation*}
    I(z)=\sum_{a\in\mathcal A(z)}F_z(a).
\end{equation*}
We sample seeds from a learned distribution and obtain admissible chains through the manifold solver. Each chain's contribution is weighted by an estimate of its reciprocal sampling probability. All identities and expectations below apply componentwise to the RGB channels.

A seed consists of a descriptor $d=(n,t_{1:n})$ and an initial direction
$\omega\in\mathbb S^2$. The seed space is
\begin{equation*}
    \mathcal Y=
    \bigsqcup_{n=1}^{N_{\max}}
    \bigsqcup_{t_{1:n}\in\{\mathrm R,\mathrm T\}^{n}}
    \bigl(\{(n,t_{1:n})\}\times\mathbb S^2\bigr).
\end{equation*}
Conditioned on $x$, we sample seeds from the normalized proposal
\begin{equation}
    q_\theta(d,\omega\mid x)
    =q_\theta(n\mid x)
     \prod_{i=1}^{n}q_\theta(t_i\mid t_{<i},n,x)
     g_\theta(\omega\mid d,x).
    \label{eq:problem-factorization}
\end{equation}
This follows MPG's length--type--direction factorization. We parameterize
the distributions with a neural network and model the R/T sequence
autoregressively.

Let $\Phi_z(d,\omega)$ denote the output of the deterministic manifold
solver. For the estimator, we define
\begin{equation*}
    \bar\Phi_z(d,\omega)=
    \begin{cases}
        a,&\Phi_z(d,\omega)=a\ \text{and}\ d(a)=d,\\
        \varnothing,&\text{otherwise}.
    \end{cases}
\end{equation*}
The outcome $\varnothing$ contributes zero to the estimator. It includes
both unsuccessful solves and solutions whose descriptors differ from
the sampled descriptor. Let $\nu$ be the product of counting measure over descriptors and surface measure on $\mathbb S^2$. The probability of obtaining chain $a$ is
\begin{equation*}
    Q_\theta(a\mid z)
    =\int_{\bar\Phi_z^{-1}(\{a\})}
    q_\theta(d,\omega\mid\psi(z))\,\mathrm d\nu(d,\omega).
\end{equation*}
The integral covers all seeds mapped to $a$ by $\bar\Phi_z$.
Evaluating the proposal density at a single seed does not give
this probability.

We train the network offline using a fixed dataset of chains obtained
by the manifold solver,
\begin{equation*}
  \mathcal R=\{(x_b,n_b,t_{b,1:n_b},\omega_b,e_b)\}_{b=1}^{B},
\end{equation*}
where $\omega_b$ is the unit direction from the receiver to the first vertex of the converged chain, and $e_b\ge0$ is a recorded scalar energy proxy. \Cref{app:implementation-details} describes the data collection procedure and the definition of $e_b$. Neither the target direction $\omega_b$ nor data from reference images are included in $x_b$. The network parameters remain fixed during rendering. Data collection and training costs are excluded from the comparison at equal SPP.

We assume that every contributing chain has positive sampling probability and that $\sum_{a\in\mathcal A(z)}\|F_z(a)\|<\infty$.
The analysis uses independent trials under a fixed state $z$ and an exact test for whether two solutions represent the same chain.
Appendix~\ref{app:ideal-estimator} gives the full assumptions and the proof of unbiasedness. Section~\ref{sec:mpgd-practical} describes the implementation used in the experiments.

%% file: sections/algorithm_design.tex
\section{Algorithm Design}
\label{sec:algorithm-design}

\method{} follows MPG's factorization into chain length, R/T sequence,
and direction, with two changes. First, we replace MPG's local tables
of complete R/T sequences and directional kernels centered on nearby
samples with a neural network of fixed size. The network models the
R/T sequence autoregressively and predicts a directional distribution
conditioned on the chain length and complete sequence. Second, MPG
holds only the chain length fixed during repeated trials and resamples
both the R/T sequence and direction. We hold both the length and sequence
fixed and resample only the direction. The resulting estimate of the
reciprocal conditional probability is divided by the exact joint
probability of the length and sequence.

\subsection{Autoregressive descriptor and directional density}
\label{sec:mpgd-proposal}

From Equation~\ref{eq:problem-factorization}, the probability of a descriptor $d=(n,t_{1:n})$ is
\begin{equation*}
  \pi_\theta(d\mid x)
  =q_\theta(n\mid x)\prod_{i=1}^{n}
    q_\theta(t_i\mid t_{<i},n,x).
\end{equation*}
The joint proposal can therefore be written as $q_\theta(d,\omega\mid x)=\pi_\theta(d\mid x)
g_\theta(\omega\mid d,x)$.
Sampling starts with all descriptor tokens masked. The network first
predicts the distribution over chain lengths. After sampling $n$, it
generates the R/T sequence in order, conditioning each prediction on
the preceding symbols. The descriptor probability $\pi_\theta(d\mid x)$ is evaluated exactly from the length probability and the conditional probabilities used during sampling.

Once the descriptor is complete, the network predicts a mixture of the
vMF densities defined in Equation~\ref{eq:bg-vmf}:
\begin{equation*}
    g_\theta(\omega\mid d,x)
    =\sum_{j=1}^{K}\rho_j
     \operatorname{vMF}(\omega;\mu_j,\kappa_j).
\end{equation*}
The number of components $K$ is fixed. The weights satisfy
$\rho_j\ge0$ and $\sum_j\rho_j=1$, the mean directions satisfy
$\mu_j\in\mathbb S^2$, and the concentrations satisfy $\kappa_j\ge0$.
All mixture parameters depend on $d$ and $x$.

MPG also conditions its directional mixtures on the chain type, but
constructs them from nearby records. Our network predicts the mixture
parameters directly from the descriptor and scene features.
Appendix~\ref{app:vmf-details} gives the numerical details of density
evaluation and sampling. The network architecture and the encoding
of features and descriptor tokens are described in
Appendix~\ref{app:implementation-details}.

\subsection{Offline objective}
\label{sec:mpgd-training}

Training minimizes the negative log likelihood of the chain length,
the R/T sequence, and the direction. For the R/T sequence, we sample
a prefix uniformly and multiply the loss for predicting the next symbol by the sequence length, giving an unbiased estimate of the loss over the complete sequence. Each record is weighted using a logarithmic transformation of its recorded energy to moderate the influence of large values. The training records are collected by the manifold integrator using online MPG, with each direction target pointing from the receiver to the first vertex of the converged chain. Reference images are not used for training. Appendix~\ref{app:implementation-details} provides the full objective, the definition of the recorded energy, and the training procedure. The trained model remains fixed during evaluation.

\subsection{Factoring type mass from re-hit trials}
\label{sec:mpgd-chain-space}

For a chain $a$ with descriptor $d(a)$, the probability of obtaining $a$ when its descriptor is held fixed is
\begin{equation*}
    R_\theta(a\mid d(a),z)
    =
    \int_{\mathbb S^2}
    \ind\{\bar\Phi_z(d(a),\omega)=a\}
    g_\theta(\omega\mid d(a),\psi(z))\,\mathrm d\omega.
\end{equation*}
Writing $d(a)=(n,t_{1:n})$ and $x=\psi(z)$, the chain probability is
$Q_\theta(a\mid z)=q_Nq_TR_\theta$, where
\begin{equation*}
    q_N=q_\theta(n\mid x),
    \qquad
    q_T=\prod_{i=1}^{n}q_\theta(t_i\mid t_{<i},n,x).
\end{equation*}
Both $q_N$ and $q_T$ can be evaluated exactly. SMS estimates reciprocal probabilities by counting independent trials until the same chain is obtained again. MPG evaluates $q_N$ separately and holds the length fixed while resampling the R/T sequence and direction on each trial \citep[Eq.~15]{fan2023manifold_path_guiding}.
We also hold the R/T sequence fixed and resample only the direction.
The trial count estimates $\frac{1}{R_\theta}$. Dividing this count by $q_Nq_T$ gives an estimate of $\frac{1}{Q_\theta(a\mid z)}$.

\begin{proposition}[Effect of fixing the R/T sequence]
\label{prop:type-stratification}
Fix a chain $a$ and state $z$, and use the same proposal and solver in
both constructions. Let $q_T\in(0,1]$ be the conditional probability
of its R/T sequence and let
$R=R_\theta(a\mid d(a),z)\in(0,1]$.
With independent trials and exact comparison of chain solutions,
the number of trials required to obtain $a$ satisfies $J_n\sim\operatorname{Geom}(q_TR)$ and $J_d\sim\operatorname{Geom}(R),$
when only the length is fixed and when the complete descriptor is fixed,
respectively. Both counts include the successful trial. Then
\begin{align*}
  &\mathbb{E}[J_d]=1/R\le 1/(q_TR)=\mathbb E[J_n],\\
  &\operatorname{Var}(J_n)-\operatorname{Var}(J_d/q_T)
    =\frac{1-q_T}{q_T^2R}\ge0.
\end{align*}
For any nonnegative integer trial limit $M$,
\begin{equation*}
    \Pr[J_d>M]
    =(1-R)^M
    \le(1-q_TR)^M
    =\Pr[J_n>M].
\end{equation*}
\end{proposition}
These results follow from
$\mathbb E[J]=1/p$,$\operatorname{Var}(J)=(1-p)/p^2$, and $\Pr[J>M]=(1-p)^M$ for $J\sim\operatorname{Geom}(p)$. To estimate $I(z)$, draw
$D\sim\pi_\theta(\cdot\mid x)$ and $\Omega_0\sim g_\theta(\cdot\mid D,x)$, and set $A=\bar\Phi_z(D,\Omega_0)$. If $A\ne\varnothing$, keep $D$ and $z$ fixed and draw new directions independently until the solver returns the same chain $A$. Let $J_d$ count these additional trials, including the successful trial. The estimator is
\begin{equation}
    \widehat I_{\methodmath}
    =
    \begin{cases}
        0,&A=\varnothing,\\
        F_z(A)J_d/(q_Nq_T),&A\ne\varnothing.
    \end{cases}
    \label{eq:mpgd-unbiased-estimator}
\end{equation}
Under the assumptions stated in Appendix~\ref{app:ideal-estimator},
this estimator is unbiased for $I(z)$. The proof is given in
Appendix~\ref{app:ideal-estimator} and follows the repeated trial
argument of SMS and MPG \citep{zeltner2020specular,fan2023manifold_path_guiding}.

\subsection{Finite-work renderer}
\label{sec:mpgd-practical}

The renderer starts from state $z_0$, samples a descriptor
$D=(N,T_{1:N})$, and draws a direction from the corresponding vMF mixture. The manifold solver then produces a chain $A$, subject to the descriptor condition defined above. The initial solve also determines a state $z_\eta$ containing the sampled offset normals. Both $z_\eta$ and $D$ remain fixed during subsequent trials, which sample new directions independently. Two solutions are treated as the same chain when their shape sequences and R/T sequences agree and their first segment directions agree within a numerical tolerance. Let $J_{\rm impl}$ denote the number of trials required to obtain a match without a trial limit, including the successful trial. We set $J_{\rm impl}=\infty$ if the initial solve fails. This numerical comparison differs from the exact equality assumed in the definition of $J_d$.

Before applying the weights for chain sampling, the contribution is
$H_{z_\eta}(A) =\frac{f_{\rm recv}(z_\eta,A)G_{z_\eta}(A)}{p_{\rm off}(z_\eta,A)},$ where $f_{\rm recv}$ is the receiver BSDF factor, $G_{z_\eta}(A)$ is the remaining chain contribution, and $p_{\rm off}$ is the density of the sampled offset normals. Samples with invalid offset densities contribute zero. For finite nonnegative contributions, the implementation computes
\begin{equation}
    \widehat I_{\mathrm{fw}}
    =
    \begin{cases}
        \displaystyle
        \frac{
            \operatorname{clip}^{\mathrm{peak}}_{C_f}
            \left(
                H_{z_\eta}(A)
                \min\left\{C_w,\frac{J_{\rm impl}}{q_T}\right\}
            \right)
        }{
            \max\{\varepsilon_N,q_N\}
        },
        & A\ne\varnothing,\ J_{\rm impl}\le M,\\[1ex]
        0,&\text{otherwise}.
    \end{cases}
    \label{eq:finite-work-return}
\end{equation}
Here, $M$ limits the number of trials, $C_w$ bounds $J_{\rm impl}/q_T$,
and $\varepsilon_N$ is a positive lower bound applied to the length
probability. The operator $\operatorname{clip}^{\mathrm{peak}}_{C_f}$
rescales all RGB channels by the same factor when the largest channel
exceeds $C_f$. This rescaling precedes division by
$\max\{\varepsilon_N,q_N\}$. The pseudocode below gives the sampling
and weighting procedure. Appendix~\ref{app:finite-work-boundary}
analyzes the effect of limiting the number of trials.

At an eligible smooth receiver, the renderer evaluates the manifold
contribution alongside ordinary next event estimation and continues
the path through BSDF sampling. The emitter and BSDF strategies are
combined using MIS with the power heuristic~\citep{veach1995optimally}.
The manifold contribution is added separately and is not included in
these MIS weights. Its proposal does not use a defensive mixture.
When the routing option is enabled and manifold sampling has been invoked,
later contributions from finite emitters reached after a caustic caster
are suppressed. Contributions from the environment remain enabled.
The proof in Appendix~\ref{app:ideal-estimator} establishes unbiasedness
of the conditional chain estimator under the stated assumptions.
It does not establish unbiasedness of the complete renderer used
in the experiments.

\begin{algorithm}[H]
\caption{Finite-work \method{} sample used in the experiments}
\label{alg:mpgd-finite-work-main}
\small
\begin{algorithmic}[1]
\Require Pre-seed state $z_0$, features $x=\psi(z_0)$; $M,C_w,C_f$.
\State Sample $D=(N,T_{1:N})\sim\pi_\theta(\cdot\mid x)$; retain $q_N,q_T$.
\State Draw $\Omega_0\sim g_\theta(\cdot\mid D,x)$; solve for
$A,p_{\rm off},z_\eta$.
\If{$A=\varnothing$ or $p_{\rm off}$ is nonfinite or $p_{\rm off}\le10^{-12}$}
  \State \Return $0$
\EndIf
\For{$j=1,\ldots,M$}
  \State Draw $\Omega_j\sim g_\theta(\cdot\mid D,x)$ and solve with frozen $z_\eta$.
  \If{shape sequence, type code, and first-direction root match $A$}
    \State $J_{\rm impl}\gets j$; \textbf{break}
  \EndIf
\EndFor
\If{no match occurred} \State \Return $0$ \EndIf
\State $V\gets H_{z_\eta}(A)\min\{C_w,J_{\rm impl}/q_T\}$.
\If{$V$ is nonfinite or $\max_cV_c<0$} \State \Return $0$ \EndIf
\State Rescale all RGB channels together if $\max_cV_c>C_f$.
\State \Return $V/\max\{10^{-6},q_N\}$.
\end{algorithmic}
\end{algorithm}

%% file: sections/experiments.tex
 \section{Experiments, Results and Discussion}
  \label{sec:experiments}


 \paragraph{Experimental setup.}
  We evaluate fourteen $64\times64$ scene configurations
  at
  $\{12,24,36,48,60\}$ samples per pixel (SPP). Every
  method is compared with
  the same fixed 32,768-SPP numerical reference for each
  scene. For every
  scene--method--SPP cell, we use the same ten fixed
  runs, denoted D01--D10,
  yielding a complete $14\times6\times5\times10=4{,}
  200$-image result matrix.

  We measure full-frame mean squared error (MSE) on raw
  linear
  high-dynamic-range (HDR) RGB values before tone
  mapping. Four-channel OpenEXR
  arrays loaded by OpenCV use blue--green--red--alpha
  (BGRA) channel order. We
  discard alpha, reorder the remaining blue--green--red
  channels to RGB, and do
  not premultiply color by alpha. Images containing
  nonfinite values are
  rejected.

  Curves report the median across independent experiment runs D01--D10. The
  60-SPP annotations
  additionally report the sample standard deviation of
  those ten values using
  denominator $n-1$, while Appendix
  Table~\ref{tab:phase47-combined-mse} also
  reports the arithmetic mean. These statistics describe
  the ten fixed runs and
  are not confidence intervals.

\subsection{Comparisons}
\label{sec:comparison-methods}

The benchmark comprises 14 scenes, \textsc{Living}, \textsc{Lamp}, \textsc{Plane},
\textsc{Slabs}, \textsc{Stone}, \textsc{Pool}, \textsc{Sphere},
\textsc{Water}, \textsc{Ring-and-Slab}, \textsc{Pillars},
\textsc{Slab-$\times1$}, \textsc{Slab-$\times3$},
\textsc{Pillar-$\times3$}, and \textsc{Cornell Glass}.  We compare \method{}
with a direction-only adaptation of Neural Parametric Mixtures
(NPM)~\citep{dong2023npm}; our implementation of Manifold Path Guiding
(MPG)~\citep{fan2023manifold_path_guiding}; a fixed-chain diagnostic based on
Specular Manifold Sampling (SMS)~\citep{zeltner2020specular}; Practical Path
Guiding (PPG)~\citep{muller2017practical}; and unidirectional path tracing
(PT)~\citep{kajiya1986rendering}.  

\begin{figure}[t]
  \centering
  \includegraphics[width=0.98\linewidth]{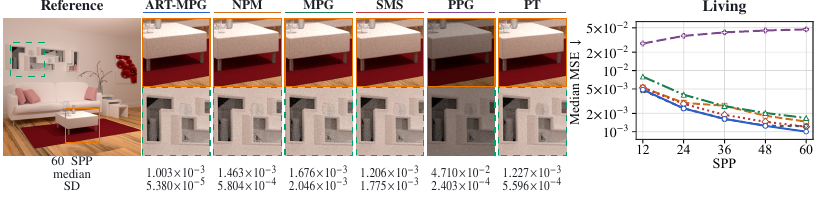}\par\vspace{0.05em}
  \includegraphics[width=0.98\linewidth]{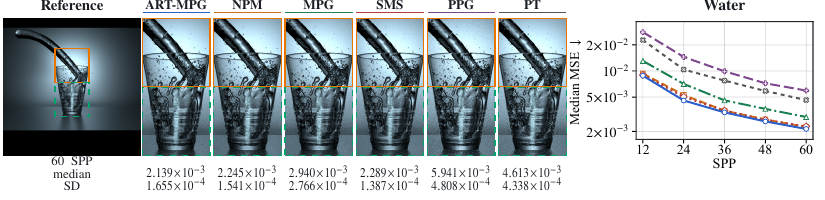}\par\vspace{0.05em}
  \includegraphics[width=0.98\linewidth]{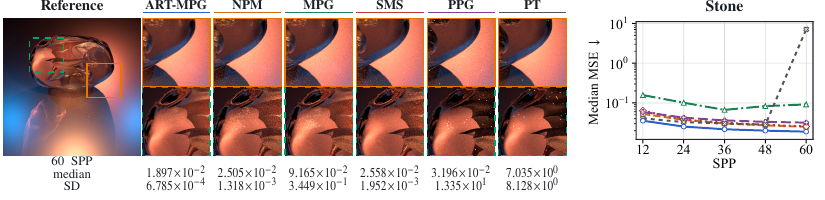}\par\vspace{0.05em}
  \includegraphics[width=0.98\linewidth]{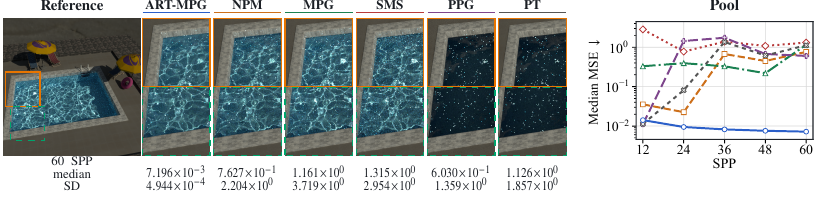}
  \caption{Equal-SPP results for \textsc{Living}, \textsc{Water},
  \textsc{Stone}, and \textsc{Pool}, from top to bottom.  Curves show median
  full-frame MSE across ten fixed runs; 60-SPP annotations additionally show
  sample standard deviation.  Image crops come from one fixed 60-SPP
  visualization run, whereas quantitative values use separate
  $64\times64$ full-frame renders.  The displayed crops are not median images.}
  \label{fig:phase47-median-main-scenes}
\end{figure}

Figure~\ref{fig:phase47-median-main-scenes} shows four representative scenes.
Each strip combines a reference overview and visual-only D01 crops from the
native $1152\times1152$ 60-SPP film with median curves computed from the
separate $64\times64$ quantitative renders.  The displayed crop illustrates
one fixed run and is not a median image.  The remaining ten scene panels appear
in Appendix~\ref{app:phase47-median-panels}.

\subsection{Evaluation}
\label{sec:quantitative-results}

\begin{table}[t]
\centering
\small
\setlength{\tabcolsep}{4.5pt}
\renewcommand{\arraystretch}{1.08}
\caption{Median full-frame linear-RGB MSE across D01--D10 for each scene at 60 SPP and $64\times64$ resolution. Lower is better.}
\label{tab:median-mse}
\resizebox{\linewidth}{!}{%
\begin{tabular}{lcccccc}
\toprule
& \multicolumn{5}{c}{Baselines} & Ours \\
\cmidrule(lr){2-6}\cmidrule(lr){7-7}
Scene & PT & PPG & SMS & MPG & NPM & ART-MPG \\
\midrule
Living
& $1.227\times10^{-3}$
& $4.710\times10^{-2}$
& $1.206\times10^{-3}$
& $1.676\times10^{-3}$
& $1.463\times10^{-3}$
& $\mathbf{1.003\times10^{-3}}$ \\

Lamp
& $4.009\times10^{-2}$
& $3.202\times10^{-2}$
& $6.472\times10^{-3}$
& $1.780\times10^{-3}$
& $2.483\times10^{-3}$
& $\mathbf{1.683\times10^{-3}}$ \\

Plane
& $2.912\times10^{-1}$
& $2.914\times10^{-1}$
& $2.992\times10^{-1}$
& $2.810\times10^{-1}$
& $2.896\times10^{-1}$
& $\mathbf{2.728\times10^{-1}}$ \\

Slabs
& $1.112\times10^{-1}$
& $2.340\times10^{1}$
& $8.095\times10^{-2}$
& $2.823\times10^{-1}$
& $7.957\times10^{-2}$
& $\mathbf{6.479\times10^{-2}}$ \\

Stone
& $7.035\times10^{0}$
& $3.196\times10^{-2}$
& $2.558\times10^{-2}$
& $9.165\times10^{-2}$
& $2.505\times10^{-2}$
& $\mathbf{1.897\times10^{-2}}$ \\

Pool
& $1.126\times10^{0}$
& $6.030\times10^{-1}$
& $1.315\times10^{0}$
& $1.161\times10^{0}$
& $7.627\times10^{-1}$
& $\mathbf{7.196\times10^{-3}}$ \\

Sphere
& $4.603\times10^{1}$
& $1.879\times10^{-1}$
& $1.016\times10^{-1}$
& $1.109\times10^{-1}$
& $8.141\times10^{-2}$
& $\mathbf{7.219\times10^{-2}}$ \\

Water
& $4.613\times10^{-3}$
& $5.941\times10^{-3}$
& $2.289\times10^{-3}$
& $2.940\times10^{-3}$
& $2.245\times10^{-3}$
& $\mathbf{2.139\times10^{-3}}$ \\

Ring--and--Slab
& $2.876\times10^{1}$
& $2.308\times10^{1}$
& $7.039\times10^{-2}$
& $4.163\times10^{-1}$
& $7.368\times10^{-2}$
& $\mathbf{5.980\times10^{-2}}$ \\

Pillars
& $9.242\times10^{-2}$
& $2.900\times10^{-1}$
& $9.190\times10^{-2}$
& $2.335\times10^{-1}$
& $8.134\times10^{-2}$
& $\mathbf{7.128\times10^{-2}}$ \\

Slab-$\times1$
& $7.981\times10^{-2}$
& $8.043\times10^{-2}$
& $5.356\times10^{-2}$
& $4.636\times10^{-2}$
& $5.258\times10^{-2}$
& $\mathbf{4.435\times10^{-2}}$ \\

Slab-$\times3$
& $2.980\times10^{-1}$
& $2.980\times10^{-1}$
& $2.618\times10^{-1}$
& $1.057\times10^{0}$
& $2.688\times10^{-1}$
& $\mathbf{2.297\times10^{-1}}$ \\

Pillar-$\times3$
& $1.024\times10^{-1}$
& $1.027\times10^{-1}$
& $9.213\times10^{-2}$
& $1.217\times10^{-1}$
& $8.291\times10^{-2}$
& $\mathbf{7.488\times10^{-2}}$ \\

Cornell Glass
& $5.445\times10^{-3}$
& $4.974\times10^{-2}$
& $\mathbf{5.435\times10^{-3}}$
& $8.887\times10^{-3}$
& $6.675\times10^{-3}$
& $6.491\times10^{-3}$ \\
\bottomrule
\end{tabular}%
}
\end{table}

At 60~SPP, \method{} has the lowest median MSE on \textbf{13/14} scenes; SMS is lower on \textsc{Cornell Glass}.  Across all five budgets,
\method{} has the lowest median in \textbf{62/70} scene--SPP cells, and its
median curve is non-increasing at every adjacent budget on \textbf{12/14}
scenes.  It has the lowest median at all five budgets on 11/14 scenes; nine
scenes satisfy both this condition and monotonicity.
Appendix~\ref{app:phase47-outcome} gives the complete rank and pairwise-win
counts, while Appendix Table~\ref{tab:phase47-combined-mse} gives every 60-SPP
median, mean, and sample standard deviation.

The only median upturns for \method{} occur on \textsc{Lamp}, from 0.001974 at
24~SPP to 0.001997 at 36~SPP, and on \textsc{Plane}, from 0.2749 to 0.2758
between 12 and 24~SPP and from 0.2725 to 0.2728 between 48 and 60~SPP.  These
small increases are retained in the appendix panels.

\subsection{Discussion}
\label{sec:results-discussion}

The evaluation supports a focused conclusion: on this controlled benchmark,
\method{} usually gives lower linear-HDR error per scheduled camera sample than
the five comparison methods.  The result is strongest at 60~SPP and is broadly
consistent across budgets, while \textsc{Cornell Glass} and the small
\textsc{Lamp}/\textsc{Plane} upturns identify the observed exceptions.

The methods operate under different learning regimes.  \method{} and NPM are trained offline and frozen before rendering, MPG adapts online, and PPG updates its guiding state during rendering; PT and the fixed-chain SMS diagnostic use
fixed configurations.  The experiment therefore controls scheduled camera samples. 

NPM accesses trainable spatial grids sparsely, whereas \method{} evaluates a dense multilayer perceptron. Execution details appear in Appendix~\ref{app:execution-environment}.

%% file: sections/limitations.tex
\section{Scope and Limitations}
\label{sec:limitations}

Our evidence supports a focused descriptive claim: over the ten runs on this controlled 14-scene benchmark, \method{}
usually has lower median linear-HDR error per scheduled camera sample than the
five comparators.  Equal SPP does not imply equal wall-clock cost; the neural
guide is fitted offline per scene, its replay/training cost is excluded.  MPG performs online fitting
inside its recorded schedule.  These are distinct operating regimes rather
than a matched-time comparison.

Quantitative results use $64\times64$
images and one fixed $32{,}768$-SPP reference seed.  Several scenes share base
geometry, and some benchmark families informed development, limiting claims of
scene-level independence or out-of-distribution generalization. These boundaries motivate further research on matched-time, reference-convergence, clamp, and topology research for different SPP observations.

%% file: sections/conclusion.tex
\section{Conclusion}
\label{sec:conclusion}
\method{} models the complete low-dimensional seed of a specular chain: its
length, ordered reflection/transmission topology, and topology-conditioned
direction.  The controlled evaluation compares \method{}, NPM, MPG,
fixed-chain SMS, PPG, and PT on 14 scenes at
$\{12,24,36,48,60\}$ SPP. Across 10 random seeds, \method{} has the
lowest median MSE in 62/70 scene--SPP cells and in 13/14 scenes at 60~SPP.  It
records 335/350 strict pairwise median wins over the five comparators across
all budgets and 68/70 at 60~SPP; its median curve is non-increasing on 12/14
scenes.  The only median upturns are small changes on \textsc{Lamp} and
\textsc{Plane}, and \textsc{Cornell Glass} is the sole 60-SPP rank exception. The results establish a sample-budget comparison, not
a matched-time comparison, and support \method{} as a compact proposal for
structured-caustic transport.

%% file: sections/reproducibility.tex
\section*{Reproducibility statement}
The supplementary artifacts record the scenes, seeds, model exports, renderer
settings, reference hashes, and per-image metric rows.  Together with
Section~\ref{sec:experiments} and the complete pseudocode and proof in
Appendix~\ref{app:ideal-estimator}, they distinguish the ideal estimator from
the finite-work profile evaluated in the paper. The artifact
contains the complete $14\times6\times5\times10=4{,}200$ rows for ART-MPG,
NPM, MPG, fixed-chain SMS, PPG, and PT at
$\{12,24,36,48,60\}$ SPP.  The reported curves are ordinary cellwise medians
without dispersion bands.  Their 60-SPP numeric annotations and the appendix
table retain raw sample standard deviations with denominator $n-1$.

\section*{AI use statement}
Generative AI assisted with pseudocode,
analysis/visualization code, and consistency checks, and performed the
disclosed unblinded visual QA.  The
authors verified AI-assisted content against source code, execution receipts,
tests, and rendered artifacts; AI was not used to synthesize measured values
or replace renderer outputs.  The authors take responsibility for the final
manuscript and artifacts.

%% file: sections/appendix.tex
\setcounter{figure}{0}
\renewcommand{\thefigure}{A\arabic{figure}}
\setcounter{table}{0}
\renewcommand{\thetable}{A\arabic{table}}
\raggedbottom

\input{sections/appendix_results}
\input{sections/appendix_implementation}
\input{sections/appendix_theory}

%% file: sections/appendix_results.tex
\section{Additional Results}
\label{app:audit-details}

\subsection{Evaluation grid}
\label{app:phase47-grid}
  The reported evaluation compares \method{}, NPM, MPG,
  fixed-chain SMS,
  PPG, and PT at $\{12,24,36,48,60\}$ SPP on fourteen
  configurations:
  \textsc{Living}, \textsc{Lamp}, \textsc{Plane},
  \textsc{Slabs},
  \textsc{Stone}, \textsc{Pool}, \textsc{Sphere},
  \textsc{Water},
  \textsc{Ring-and-Slab}, \textsc{Pillars},
  \textsc{Slab-$\times1$},
  \textsc{Slab-$\times3$}, \textsc{Pillar-$\times3$},
  and
  \textsc{Cornell Glass}. Every quantitative cell is
  full-frame MSE on raw
  linear-HDR RGB at $64\times64$ resolution against the
  same fixed
  32,768-SPP numerical reference for that scene.

The SMS curve is the hash-bound fixed-chain configuration used by this sweep:
the established \textsc{Lamp}/\textsc{Slabs} presets are retained, and the
remaining scenes use the frozen descriptor recorded by the protocol.  We
therefore interpret SMS as a fixed-chain diagnostic, not as a variable-chain
SMS result.

The ten fixed runs are denoted D01--D10. Each run
  contains the complete
  $14\times6\times5$ grid, yielding
  $14\times6\times5\times10=4{,}200$ measured images.
  The same run IDs are
  used for all methods and scenes. 

\subsection{Median outcome}
\label{app:phase47-outcome}
ART-MPG has the lowest median MSE in 62/70 scene--SPP cells, including 13/14
scenes at 60~SPP.  Its median curve is non-increasing on 12/14 scenes and
strictly rank~1 at every displayed budget on 11/14 scenes; nine scenes satisfy
both properties.  It records 335/350 strict pairwise median wins against the
five comparators across all budgets and 68/70 at 60~SPP.  \textsc{Cornell
Glass} is the sole 60-SPP rank exception.

The only ART-MPG median upturns are \textsc{Lamp}, from 0.001974 at 24~SPP to
0.001997 at 36~SPP, and \textsc{Plane}, from 0.2749 to 0.2758 between 12 and
24~SPP and from 0.2725 to 0.2728 between 48 and 60~SPP.  All remain in the panels below.  

\subsection{Complete 60-SPP descriptive summary}
\label{app:phase47-summary-table}
  Table~\ref{tab:phase47-combined-mse} reports the 60-
  SPP median, arithmetic
  mean, and sample standard deviation for every
  scene--method pair. The main
  paper's claims and curves use medians; the other two
  statistics are included
  to describe variation across D01--D10.

\clearpage
\begingroup
\iclrfinaltrue
\input{figures/phase47_d1_d10_statistics_v3/tables/combined_mse_statistics_60spp.tex}
\clearpage
\endgroup

\subsection{Remaining ten median panels}
\label{app:phase47-median-panels}
Figures~\ref{fig:phase47-median-living-lamp}--
\ref{fig:phase47-median-pillar-cornell} report the ten scenes not shown in
Figure~\ref{fig:phase47-median-main-scenes}.  Each compact strip places a
reference overview, two visual-only D01 native $288\times288$ crops per method
from a $1152\times1152$ 60-SPP film, and the quantitative curve side by side.  The crop is not
the median image.  Curve centers and numeric labels summarize D01--D10
$64\times64$ full-frame MSE: each curve shows the median only.  Each method's 60-SPP numeric annotation additionally lists
raw sample SD (denominator $n-1$).

\begin{figure}[!htbp]
  \centering
  \includegraphics[width=\linewidth]{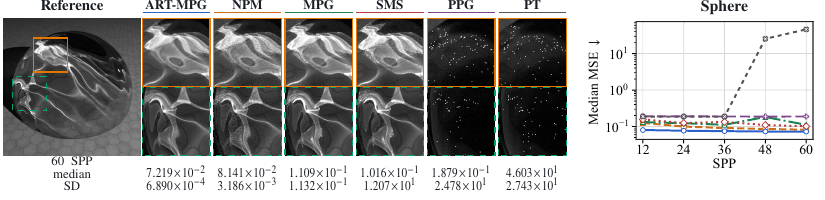}\par\vspace{0.05em}
  \includegraphics[width=\linewidth]{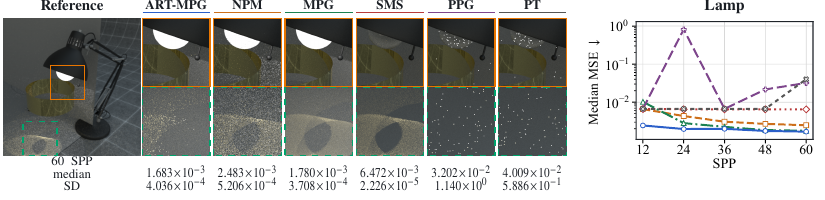}
  \caption{Median MSE summaries for
  \textsc{Sphere} (top) and \textsc{Lamp} (bottom).}
  \label{fig:phase47-median-living-lamp}
\end{figure}

\begin{figure}[p]
  \centering
  \includegraphics[width=\linewidth]{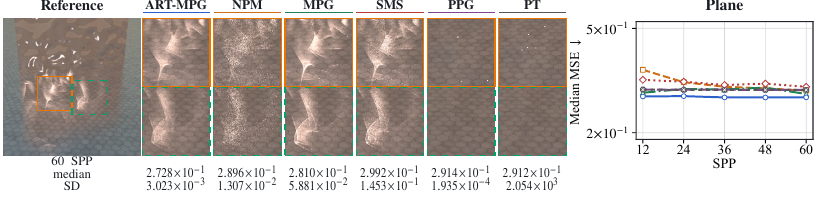}\par\vspace{0.05em}
  \includegraphics[width=\linewidth]{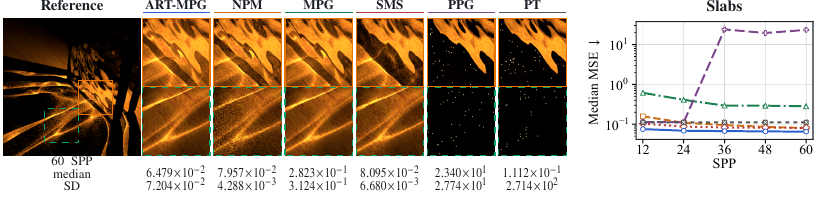}
  \caption{Median MSE summaries for
  \textsc{Plane} (top) and \textsc{Slabs} (bottom).}
  \label{fig:phase47-median-plane-slabs}
\end{figure}

\begin{figure}[p]
  \centering
  \includegraphics[width=\linewidth]{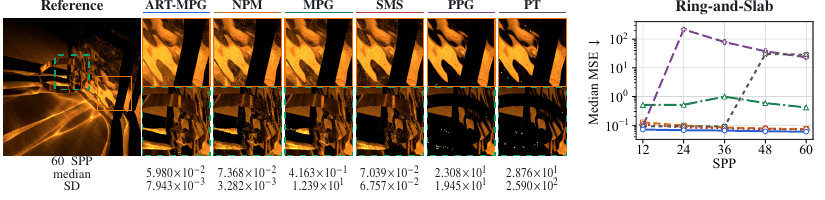}\par\vspace{0.05em}
  \includegraphics[width=\linewidth]{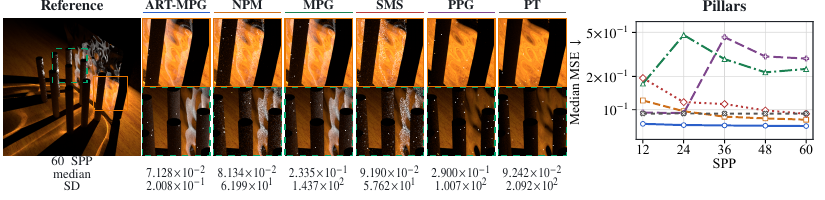}
  \caption{Median MSE summaries for
  \textsc{Ring-and-Slab} (top) and \textsc{Pillars} (bottom).}
  \label{fig:phase47-median-ring-pillars}
\end{figure}

\begin{figure}[p]
  \centering
  \includegraphics[width=\linewidth]{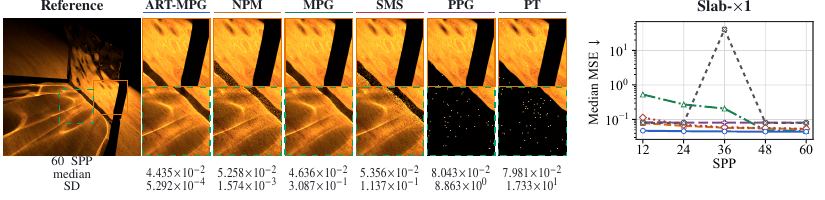}\par\vspace{0.05em}
  \includegraphics[width=\linewidth]{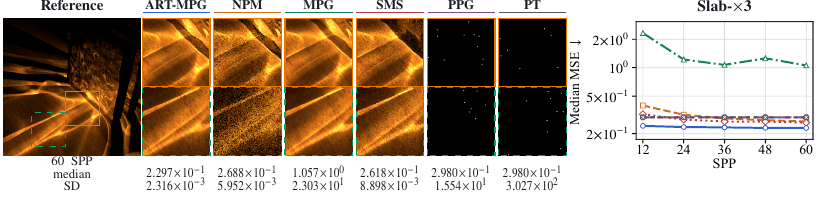}
  \caption{Median MSE summaries for
  \textsc{Slab-$\times1$} (top) and \textsc{Slab-$\times3$} (bottom).}
  \label{fig:phase47-median-slab-one-three}
\end{figure}

\begin{figure}[p]
  \centering
  \includegraphics[width=\linewidth]{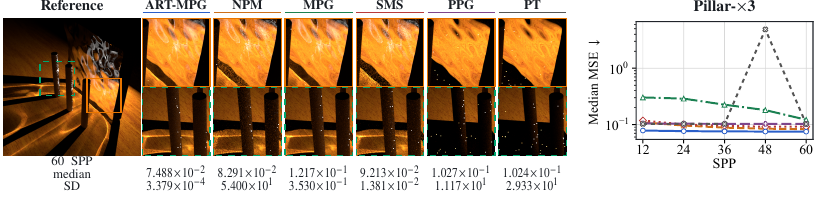}\par\vspace{0.05em}
  \includegraphics[width=\linewidth]{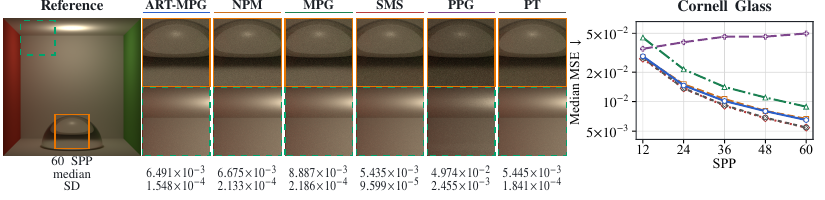}
  \caption{Median MSE summaries for
  \textsc{Pillar-$\times3$} (top) and \textsc{Cornell Glass} (bottom).}
  \label{fig:phase47-median-pillar-cornell}
\end{figure}

\clearpage

%% file: figures/phase47_d1_d10_statistics_v3/tables/combined_mse_statistics_60spp.tex
\begin{table}[p]
\centering
\footnotesize
\setlength{\tabcolsep}{4.0pt}
\renewcommand{\arraystretch}{1.05}
\caption{Combined 60-SPP 64$\times$64 full-frame linear-RGB MSE summaries across D01--D10 ($n=10$ per cell). Each method cell lists median, arithmetic mean, and sample standard deviation from top to bottom. Row-minimum medians and means are bold; SD is not ranked.}
\label{tab:phase47-combined-mse}
\resizebox{\textwidth}{!}{%
\begin{tabular}{lcccccc}
\toprule
& \multicolumn{5}{c}{Baselines} & \multicolumn{1}{c}{Ours} \\
\cmidrule(lr){2-6}\cmidrule(lr){7-7}
Scene & PT & PPG & SMS & MPG & NPM & ART-MPG \\
& \multicolumn{6}{c}{Cell order: median / mean / sample SD; MSE $\downarrow$} \\
\midrule
\rowcolor{black!6}
\textsc{Living} & \shortstack{$1.227\!\times\!10^{-3}$\\$1.423\!\times\!10^{-3}$\\$5.596\!\times\!10^{-4}$} & \shortstack{$4.710\!\times\!10^{-2}$\\$4.711\!\times\!10^{-2}$\\$2.403\!\times\!10^{-4}$} & \shortstack{$1.206\!\times\!10^{-3}$\\$1.769\!\times\!10^{-3}$\\$1.775\!\times\!10^{-3}$} & \shortstack{$1.676\!\times\!10^{-3}$\\$2.739\!\times\!10^{-3}$\\$2.046\!\times\!10^{-3}$} & \shortstack{$1.463\!\times\!10^{-3}$\\$1.593\!\times\!10^{-3}$\\$5.804\!\times\!10^{-4}$} & \shortstack{$\mathbf{1.003\!\times\!10^{-3}}$\\$\mathbf{1.008\!\times\!10^{-3}}$\\$5.380\!\times\!10^{-5}$} \\
\textsc{Lamp} & \shortstack{$4.009\!\times\!10^{-2}$\\$4.003\!\times\!10^{-1}$\\$5.886\!\times\!10^{-1}$} & \shortstack{$3.202\!\times\!10^{-2}$\\$6.495\!\times\!10^{-1}$\\$1.140\!\times\!10^{0}$} & \shortstack{$6.472\!\times\!10^{-3}$\\$6.468\!\times\!10^{-3}$\\$2.226\!\times\!10^{-5}$} & \shortstack{$1.780\!\times\!10^{-3}$\\$1.869\!\times\!10^{-3}$\\$3.708\!\times\!10^{-4}$} & \shortstack{$2.483\!\times\!10^{-3}$\\$2.493\!\times\!10^{-3}$\\$5.206\!\times\!10^{-4}$} & \shortstack{$\mathbf{1.683\!\times\!10^{-3}}$\\$\mathbf{1.776\!\times\!10^{-3}}$\\$4.036\!\times\!10^{-4}$} \\
\rowcolor{black!6}
\textsc{Plane} & \shortstack{$2.912\!\times\!10^{-1}$\\$6.499\!\times\!10^{2}$\\$2.054\!\times\!10^{3}$} & \shortstack{$2.914\!\times\!10^{-1}$\\$2.914\!\times\!10^{-1}$\\$1.935\!\times\!10^{-4}$} & \shortstack{$2.992\!\times\!10^{-1}$\\$3.396\!\times\!10^{-1}$\\$1.453\!\times\!10^{-1}$} & \shortstack{$2.810\!\times\!10^{-1}$\\$2.982\!\times\!10^{-1}$\\$5.881\!\times\!10^{-2}$} & \shortstack{$2.896\!\times\!10^{-1}$\\$2.903\!\times\!10^{-1}$\\$1.307\!\times\!10^{-2}$} & \shortstack{$\mathbf{2.728\!\times\!10^{-1}}$\\$\mathbf{2.721\!\times\!10^{-1}}$\\$3.023\!\times\!10^{-3}$} \\
\textsc{Slabs} & \shortstack{$1.112\!\times\!10^{-1}$\\$9.797\!\times\!10^{1}$\\$2.714\!\times\!10^{2}$} & \shortstack{$2.340\!\times\!10^{1}$\\$2.464\!\times\!10^{1}$\\$2.774\!\times\!10^{1}$} & \shortstack{$8.095\!\times\!10^{-2}$\\$8.245\!\times\!10^{-2}$\\$6.680\!\times\!10^{-3}$} & \shortstack{$2.823\!\times\!10^{-1}$\\$4.073\!\times\!10^{-1}$\\$3.124\!\times\!10^{-1}$} & \shortstack{$7.957\!\times\!10^{-2}$\\$\mathbf{8.116\!\times\!10^{-2}}$\\$4.288\!\times\!10^{-3}$} & \shortstack{$\mathbf{6.479\!\times\!10^{-2}}$\\$9.854\!\times\!10^{-2}$\\$7.204\!\times\!10^{-2}$} \\
\rowcolor{black!6}
\textsc{Stone} & \shortstack{$7.035\!\times\!10^{0}$\\$7.849\!\times\!10^{0}$\\$8.128\!\times\!10^{0}$} & \shortstack{$3.196\!\times\!10^{-2}$\\$7.869\!\times\!10^{0}$\\$1.335\!\times\!10^{1}$} & \shortstack{$2.558\!\times\!10^{-2}$\\$2.621\!\times\!10^{-2}$\\$1.952\!\times\!10^{-3}$} & \shortstack{$9.165\!\times\!10^{-2}$\\$2.728\!\times\!10^{-1}$\\$3.449\!\times\!10^{-1}$} & \shortstack{$2.505\!\times\!10^{-2}$\\$2.529\!\times\!10^{-2}$\\$1.318\!\times\!10^{-3}$} & \shortstack{$\mathbf{1.897\!\times\!10^{-2}}$\\$\mathbf{1.931\!\times\!10^{-2}}$\\$6.785\!\times\!10^{-4}$} \\
\textsc{Pool} & \shortstack{$1.126\!\times\!10^{0}$\\$1.722\!\times\!10^{0}$\\$1.857\!\times\!10^{0}$} & \shortstack{$6.030\!\times\!10^{-1}$\\$1.137\!\times\!10^{0}$\\$1.359\!\times\!10^{0}$} & \shortstack{$1.315\!\times\!10^{0}$\\$2.641\!\times\!10^{0}$\\$2.954\!\times\!10^{0}$} & \shortstack{$1.161\!\times\!10^{0}$\\$2.707\!\times\!10^{0}$\\$3.719\!\times\!10^{0}$} & \shortstack{$7.627\!\times\!10^{-1}$\\$1.425\!\times\!10^{0}$\\$2.204\!\times\!10^{0}$} & \shortstack{$\mathbf{7.196\!\times\!10^{-3}}$\\$\mathbf{7.214\!\times\!10^{-3}}$\\$4.944\!\times\!10^{-4}$} \\
\rowcolor{black!6}
\textsc{Sphere} & \shortstack{$4.603\!\times\!10^{1}$\\$4.691\!\times\!10^{1}$\\$2.743\!\times\!10^{1}$} & \shortstack{$1.879\!\times\!10^{-1}$\\$1.904\!\times\!10^{1}$\\$2.478\!\times\!10^{1}$} & \shortstack{$1.016\!\times\!10^{-1}$\\$3.945\!\times\!10^{0}$\\$1.207\!\times\!10^{1}$} & \shortstack{$1.109\!\times\!10^{-1}$\\$1.484\!\times\!10^{-1}$\\$1.132\!\times\!10^{-1}$} & \shortstack{$8.141\!\times\!10^{-2}$\\$8.188\!\times\!10^{-2}$\\$3.186\!\times\!10^{-3}$} & \shortstack{$\mathbf{7.219\!\times\!10^{-2}}$\\$\mathbf{7.229\!\times\!10^{-2}}$\\$6.890\!\times\!10^{-4}$} \\
\textsc{Water} & \shortstack{$4.613\!\times\!10^{-3}$\\$4.632\!\times\!10^{-3}$\\$4.338\!\times\!10^{-4}$} & \shortstack{$5.941\!\times\!10^{-3}$\\$5.824\!\times\!10^{-3}$\\$4.808\!\times\!10^{-4}$} & \shortstack{$2.289\!\times\!10^{-3}$\\$2.270\!\times\!10^{-3}$\\$1.387\!\times\!10^{-4}$} & \shortstack{$2.940\!\times\!10^{-3}$\\$2.991\!\times\!10^{-3}$\\$2.766\!\times\!10^{-4}$} & \shortstack{$2.245\!\times\!10^{-3}$\\$2.298\!\times\!10^{-3}$\\$1.541\!\times\!10^{-4}$} & \shortstack{$\mathbf{2.139\!\times\!10^{-3}}$\\$\mathbf{2.171\!\times\!10^{-3}}$\\$1.655\!\times\!10^{-4}$} \\
\rowcolor{black!6}
\textsc{Ring--and--Slab} & \shortstack{$2.876\!\times\!10^{1}$\\$1.085\!\times\!10^{2}$\\$2.590\!\times\!10^{2}$} & \shortstack{$2.308\!\times\!10^{1}$\\$2.080\!\times\!10^{1}$\\$1.945\!\times\!10^{1}$} & \shortstack{$7.039\!\times\!10^{-2}$\\$9.742\!\times\!10^{-2}$\\$6.757\!\times\!10^{-2}$} & \shortstack{$4.163\!\times\!10^{-1}$\\$5.683\!\times\!10^{0}$\\$1.239\!\times\!10^{1}$} & \shortstack{$7.368\!\times\!10^{-2}$\\$7.356\!\times\!10^{-2}$\\$3.282\!\times\!10^{-3}$} & \shortstack{$\mathbf{5.980\!\times\!10^{-2}}$\\$\mathbf{6.366\!\times\!10^{-2}}$\\$7.943\!\times\!10^{-3}$} \\
\textsc{Pillars} & \shortstack{$9.242\!\times\!10^{-2}$\\$7.273\!\times\!10^{1}$\\$2.092\!\times\!10^{2}$} & \shortstack{$2.900\!\times\!10^{-1}$\\$5.739\!\times\!10^{1}$\\$1.007\!\times\!10^{2}$} & \shortstack{$9.190\!\times\!10^{-2}$\\$1.832\!\times\!10^{1}$\\$5.762\!\times\!10^{1}$} & \shortstack{$2.335\!\times\!10^{-1}$\\$4.628\!\times\!10^{1}$\\$1.437\!\times\!10^{2}$} & \shortstack{$8.134\!\times\!10^{-2}$\\$1.968\!\times\!10^{1}$\\$6.199\!\times\!10^{1}$} & \shortstack{$\mathbf{7.128\!\times\!10^{-2}}$\\$\mathbf{1.345\!\times\!10^{-1}}$\\$2.008\!\times\!10^{-1}$} \\
\rowcolor{black!6}
\textsc{Slab-$\times 1$} & \shortstack{$7.981\!\times\!10^{-2}$\\$1.338\!\times\!10^{1}$\\$1.733\!\times\!10^{1}$} & \shortstack{$8.043\!\times\!10^{-2}$\\$4.282\!\times\!10^{0}$\\$8.863\!\times\!10^{0}$} & \shortstack{$5.356\!\times\!10^{-2}$\\$1.032\!\times\!10^{-1}$\\$1.137\!\times\!10^{-1}$} & \shortstack{$4.636\!\times\!10^{-2}$\\$1.478\!\times\!10^{-1}$\\$3.087\!\times\!10^{-1}$} & \shortstack{$5.258\!\times\!10^{-2}$\\$5.252\!\times\!10^{-2}$\\$1.574\!\times\!10^{-3}$} & \shortstack{$\mathbf{4.435\!\times\!10^{-2}}$\\$\mathbf{4.435\!\times\!10^{-2}}$\\$5.292\!\times\!10^{-4}$} \\
\textsc{Slab-$\times 3$} & \shortstack{$2.980\!\times\!10^{-1}$\\$1.439\!\times\!10^{2}$\\$3.027\!\times\!10^{2}$} & \shortstack{$2.980\!\times\!10^{-1}$\\$5.213\!\times\!10^{0}$\\$1.554\!\times\!10^{1}$} & \shortstack{$2.618\!\times\!10^{-1}$\\$2.619\!\times\!10^{-1}$\\$8.898\!\times\!10^{-3}$} & \shortstack{$1.057\!\times\!10^{0}$\\$8.427\!\times\!10^{0}$\\$2.303\!\times\!10^{1}$} & \shortstack{$2.688\!\times\!10^{-1}$\\$2.697\!\times\!10^{-1}$\\$5.952\!\times\!10^{-3}$} & \shortstack{$\mathbf{2.297\!\times\!10^{-1}}$\\$\mathbf{2.303\!\times\!10^{-1}}$\\$2.316\!\times\!10^{-3}$} \\
\rowcolor{black!6}
\textsc{Pillar-$\times 3$} & \shortstack{$1.024\!\times\!10^{-1}$\\$1.650\!\times\!10^{1}$\\$2.933\!\times\!10^{1}$} & \shortstack{$1.027\!\times\!10^{-1}$\\$3.635\!\times\!10^{0}$\\$1.117\!\times\!10^{1}$} & \shortstack{$9.213\!\times\!10^{-2}$\\$9.581\!\times\!10^{-2}$\\$1.381\!\times\!10^{-2}$} & \shortstack{$1.217\!\times\!10^{-1}$\\$2.380\!\times\!10^{-1}$\\$3.530\!\times\!10^{-1}$} & \shortstack{$8.291\!\times\!10^{-2}$\\$1.716\!\times\!10^{1}$\\$5.400\!\times\!10^{1}$} & \shortstack{$\mathbf{7.488\!\times\!10^{-2}}$\\$\mathbf{7.478\!\times\!10^{-2}}$\\$3.379\!\times\!10^{-4}$} \\
\textsc{Cornell Glass} & \shortstack{$5.445\!\times\!10^{-3}$\\$5.472\!\times\!10^{-3}$\\$1.841\!\times\!10^{-4}$} & \shortstack{$4.974\!\times\!10^{-2}$\\$4.897\!\times\!10^{-2}$\\$2.455\!\times\!10^{-3}$} & \shortstack{$\mathbf{5.435\!\times\!10^{-3}}$\\$\mathbf{5.465\!\times\!10^{-3}}$\\$9.599\!\times\!10^{-5}$} & \shortstack{$8.887\!\times\!10^{-3}$\\$8.908\!\times\!10^{-3}$\\$2.186\!\times\!10^{-4}$} & \shortstack{$6.675\!\times\!10^{-3}$\\$6.664\!\times\!10^{-3}$\\$2.133\!\times\!10^{-4}$} & \shortstack{$6.491\!\times\!10^{-3}$\\$6.475\!\times\!10^{-3}$\\$1.548\!\times\!10^{-4}$} \\
\bottomrule
\end{tabular}%
}
\end{table}

%% file: sections/appendix_implementation.tex
\section{Implementation Details}
\label{app:implementation}

\setcounter{table}{0}
\renewcommand{\thetable}{B\arabic{table}}

The public method name is \method{}; immutable manifests retain historical
method keys for provenance, but those keys do not denote a different
algorithm.  All image metrics are computed on raw linear-RGB OpenEXR images
rather than display PNGs.

\subsection{Feature construction, objective, and solver details}
\label[appendix]{app:implementation-details}

The 22-D feature vector is, in order: receiver position (3), sampled-emitter
position (3), their displacement $\delta$ (3), distance (1), receiver normal
$n_D$ (3), emitter normal $n_L$ (3), emitter direction $d_L$ (3), and
$\hat\delta^\top n_D$, $-\hat\delta^\top n_L$, and
$\hat\delta^\top d_L$.  Raw positions and distance are not normalized.  The
renderer sets $\hat\delta=0$ at its $10^{-6}$ zero-distance guard; the loader
uses $\delta/\max(\lVert\delta\rVert,10^{-6})$.  Ten length logits reserve
class zero; runtime sampling drops that class and renormalizes lengths 1--9.
Each R/T sigmoid probability is clipped only to machine epsilon, consistently
in sampling and mass evaluation.

Replay consists of successful solved chains exported by the online MPG-guided
manifold integrator and then aggregated and deduplicated.  For each record,
$\omega_b$ is the normalized direction from the receiver to the first vertex
of the converged chain, rather than the proposal direction before projection;
this follows MPG's sample-centred directional construction.  Ten scenes reuse
hash-verified legacy exports, while four added scenes use fresh five-seed
exports; acquisition budgets are therefore scene-specific rather than one
uniform fourteen-scene schedule.  If $G_z(a)$ denotes the RGB path-contribution
field before the receiver bidirectional scattering distribution function
(BSDF), $\widehat r^{-1}$ the export-time reciprocal
basin estimate, $p_{\rm off}$ the offset-normal density, and $q_{N,\rm exp}$ the
export length probability, the stored scalar is operationally
\begin{equation}
  e_b=\operatorname{mean}_{\rm RGB}\!\left(
  G_z(a)\widehat r^{-1}/p_{\rm off}\right)/q_{N,\rm exp}.
\end{equation}
Thus $e_b$ is an importance proxy rather than the full image contribution; in
particular, it omits the receiver BSDF.  The Python loader preserves this field
and applies only the following log transform.  With
$w_b=1+\log(1+\max(e_b,0))$, the exact per-record losses are
\begin{align}
  \ell_n(b)&=-\log q_\theta(n_b\mid x_b),\\
  \ell_t(b,U)&=-n_b\log q_\theta
  (t_{b,U+1}\mid t_{b,1:U},n_b,x_b),\quad
  U\sim\operatorname{Unif}\{0,\ldots,n_b-1\},
  \label{eq:mpgd-type-loss}\\
  \ell_\omega(b)&=-\log g_\theta(\omega_b\mid d_b,x_b).
\end{align}
Training minimizes
\begin{equation}
  \mathcal L=
  \frac{\sum_b w_b[\ell_n(b)+\ell_t(b,U_b)+
  \lambda_\omega\ell_\omega(b)]}{\sum_b w_b}.
  \label{eq:mpgd-training-loss}
\end{equation}
Because the sampled transition is multiplied by $n_b$, its expectation is the
sum of all next-bit NLL terms in that record.

The manifold solve uses angle-difference constraints, at most 20 damped Newton
iterations, and residual threshold $10^{-4}$.  Failed surface reprojections
halve the step.  The practical same-chain predicate is intentionally weaker
than the theorem's canonical identity: it requires the primary shape sequence,
packed type code, and first solved directions with
$|\omega_1'^\top\omega_1-1|<10^{-4}$.

\begin{algorithm}[H]
\caption{Offline prefix-autoregressive training for \method{}}
\label{alg:mpgd-training}
\begin{algorithmic}[1]
\Require Fixed replay $\mathcal R$, epoch count $E$, and direction coefficient
$\lambda_\omega$.
\State Split record groups into training/validation; initialize model and AdamW.
\State Set $w_b\gets1+\log(1+\max(e_b,0))$.
\For{$e=1,\ldots,E$}
  \For{each minibatch $\mathcal B$}
    \State Evaluate length NLL from the all-mask state.
    \State Sample one prefix per record and evaluate the scaled next-bit NLL.
    \State Reveal the completed descriptor and evaluate vMF-mixture NLL.
    \State Form objective~\ref{eq:mpgd-training-loss}; clip gradients; update.
  \EndFor
  \State Enumerate all validation prefixes; retain minimum validation loss.
\EndFor
\State \Return the frozen proposal $q_\theta$.
\end{algorithmic}
\end{algorithm}

\subsection{vMF density and sampling}
\label{app:vmf-details}

For $\kappa>0$, the vMF normalizer on $\mathbb S^2$ is
$\kappa/(4\pi\sinh\kappa)$, with continuous limit $1/(4\pi)$ at
$\kappa=0$~\citep{mardia2000directional}.  Training evaluates the mixture NLL
with log-sum-exp; runtime uses a stable log normalizer and linearly sums the
weighted component densities.  To sample one component, it draws
\begin{equation}
  w=1+\kappa^{-1}\log\!\left(u_1+(1-u_1)e^{-2\kappa}\right),
  \qquad \varphi=2\pi u_2,
\end{equation}
constructs $(\sqrt{1-w^2}\cos\varphi,\sqrt{1-w^2}\sin\varphi,w)$, and rotates
the north pole to $\mu$~\citep{wood1994vmf}.  Runtime guards the uniform draws
away from the inverse-CDF endpoints; for very small $\kappa$ it samples the
sphere uniformly.  Training forms each mean as
$r_\mu/\max(\lVert r_\mu\rVert,10^{-6})$.  Runtime uses the corresponding unit
direction when $\lVert r_\mu\rVert^2>10^{-12}$ and otherwise falls back to
$+z$.  Softmax mixture weights and clamped concentrations are shared across
training NLL, runtime sampling, and density evaluation; the mean conventions
therefore agree outside the stated degenerate region.

\subsection{Recorded reproducibility environment}
\label{app:execution-environment}

The execution receipts bind the worker Python environment, renderer binaries,
and OpenCV 4.11.0 EXR decoder.  The reported D01--D10 grid was rendered on one
host with 80 one-thread workers.  Binary, plugin, protocol, model, reference,
and output hashes define the reproducibility boundary.  Quantitative metrics
use $64\times64$ linear-RGB images.

\subsection{Result and figure provenance}
\label{app:spp-mse}

The result artifact contains the complete
$14\times6\times5\times10=4{,}200$-row matrix for D01--D10.  No per-scene
substitution, smoothing, interpolation, monotone fitting, exclusion, or
replacement is applied.  The statistics artifact computes ordinary cellwise
medians, arithmetic means, and sample standard deviations with denominator
$n-1$.  Publication curves show the median centerline without a dispersion
band; 60-SPP annotations and Table~\ref{tab:phase47-combined-mse} retain the
sample SD descriptively.  Source manifests and completion receipts bind the
run rows, statistics, and publication figures by SHA-256.  These summaries
describe dispersion across the ten fixed runs and are not uncertainty
intervals.

%% file: sections/appendix_theory.tex
\section{Theory}
\label{app:theory}

\subsection{Ideal conditional estimator}
\label[appendix]{app:ideal-estimator}
Assume that $\mathcal A(z)$ is countable,
$\sum_a\lVert F_z(a)\rVert<\infty$, and every contributing chain has
$q_Nq_TR_\theta>0$.  The complete state $z$ is exogenous to the sampled
descriptor and direction and is frozen across a primary draw and all retries;
$\bar\Phi_z$ is deterministic; and retries are independent with exact
canonical-chain identity.  This theorem directly covers deterministic offsets.
The executed glossy-offset routine instead draws offset normals after a
seed-dependent preliminary solve and corrects by their density; a proof for
that augmented seed space is not asserted here.

This result is the descriptor-conditioned specialization of the SMS/MPG
geometric reciprocal identity~\citep{zeltner2020specular,
fan2023manifold_path_guiding}.
\begin{theorem}[Descriptor-conditioned SMS/MPG re-hit identity]
\label{thm:conditional-rehit}
Draw $D\sim\pi_\theta(\cdot\mid\psi(z))$ and
$\Omega_0\sim g_\theta(\cdot\mid D,\psi(z))$, and set
$A=\bar\Phi_z(D,\Omega_0)$.  On success, keep $D,z$ fixed and redraw directions
until $\bar\Phi_z(D,\Omega')=A$.  If $J_d\in\{1,2,\ldots\}$ counts trials
including the successful re-hit, then the estimator in
Equation~\ref{eq:mpgd-unbiased-estimator} is unbiased for $I(z)$.
\end{theorem}

\begin{proof}
Conditioned on $A=a$ and $D=d(a)$, each retry succeeds with probability
$R_\theta(a\mid d(a),z)$, so
$J_d\sim\operatorname{Geom}(R_\theta)$ and
$\mathbb E[J_d\mid A=a]=1/R_\theta$.  Since
$Q_\theta(a\mid z)=q_Nq_TR_\theta$,
\begin{align*}
  \mathbb E[\widehat I_{\methodmath}\mid z]
  &=\sum_d\pi_\theta(d\mid\psi(z))
    \sum_{a:d(a)=d}R_\theta(a\mid d,z)
    \frac{F_z(a)}{\pi_\theta(d\mid\psi(z))}
    \frac{1}{R_\theta(a\mid d,z)}\\
  &=\sum_{a\in\mathcal A(z)}F_z(a)=I(z).
\end{align*}
Absolute summability justifies exchanging the sums and expectation.
\end{proof}

Proposition~\ref{prop:type-stratification} isolates our added length-plus-type
stratification relative to MPG's length-only form; repeated-trial unbiasedness
itself is inherited rather than claimed as new.

\begin{algorithm}[H]
\caption{Ideal uncapped conditional estimator}
\label{alg:mpgd-ideal}
\begin{algorithmic}[1]
\Require Frozen state $z$, features $x=\psi(z)$, normalized proposal
$q_\theta(d,\omega\mid x)$, filtered map $\bar\Phi_z$, and exact identity.
\State Draw $D\sim\pi_\theta(\cdot\mid x)$ and
$\Omega_0\sim g_\theta(\cdot\mid D,x)$.
\State Set $A\gets\bar\Phi_z(D,\Omega_0)$; if $A=\varnothing$, \Return $0$.
\State Set $J_d\gets0$.
\Repeat
  \State $J_d\gets J_d+1$; independently draw
  $\Omega'\sim g_\theta(\cdot\mid D,x)$.
  \State Set $A'\gets\bar\Phi_z(D,\Omega')$.
\Until{$A'$ has the exact canonical identity of $A$}
\State \Return $F_z(A)J_d/\pi_\theta(D\mid x)$.
\end{algorithmic}
\end{algorithm}

\subsection{Finite-work estimator boundary}
\label{app:finite-work-boundary}
Algorithm~\ref{alg:mpgd-finite-work-main} gives the implementation-matched
sampling procedure, including the offset-density and contribution guards.
The experiment profile fixes $M=2048$, $C_w=512$, and $C_f=32$.

Conditional on a successful primary solve and fixed practical target, let
$R_{\rm impl}$ be the probability that one independent retry satisfies the
implementation predicate.  Then
$J_{\rm impl}\sim\operatorname{Geom}(R_{\rm impl})$ and
$\widetilde J_M=J_{\rm impl}\ind\{J_{\rm impl}\le M\}$ obeys
\begin{equation*}
  \mathbb E[\widetilde J_M\mid A=a]
  =\frac{1-(M+1)(1-R_{\rm impl})^M
  +M(1-R_{\rm impl})^{M+1}}{R_{\rm impl}}.
\end{equation*}
Retry truncation alone therefore retains the fraction
$1-(1-R_{\rm impl})^M(1+MR_{\rm impl})$ relative to the practical
equivalence class.  This is distinct from the exact-identity theorem.  The clamps and
$q_N$ floor have distribution-dependent effects and are fixed parts of the
evaluated profile; they are not assumptions of Theorem~\ref{thm:conditional-rehit}.

\section{Extensions and Future Work}
\label{app:future-work}

\paragraph{Iterative guide refinement.}
ART-MPG learns from a fixed offline replay dataset. An
extension is to alternate sampling with the current guide
and teacher-guided proposal updates, drawing inspiration
from recursive on-policy distillation
\citep{yin2026recursive}. This would require a
transport-specific teacher and an objective that accounts
for the changing sampling distribution. The proposal would
remain fixed during each primary draw and its re-hit trials.

\paragraph{Adaptive selection of training scenes.}
Training across scenes raises the question of how to allocate
data collection toward configurations where the guide
performs poorly. One possible formulation is a zero-sum
interaction between a proposal learner and a scene selector,
with opposing objectives based on rendering error under a
fixed computational budget. If this interaction admits an
episodic Markov-game formulation, Bayesian learning methods
such as those of \citet{yueh2026bayesian} could provide a
starting point for studying exploration under uncertainty.

\paragraph{Decentralized allocation of sampling effort.}
A further extension is to assign separate proposal agents
to different transport strata and let them allocate effort
under shared computational constraints. If their utilities
depend on one another's allocations, this yields a strategic
resource-allocation problem. Generalized quantal response
equilibria and their decentralized learning algorithms
\citep{zhao2025efficient} offer a possible framework for
regularized allocation decisions. Establishing suitable
utilities and relating equilibrium allocations to rendering
variance and runtime remain open questions

%% file: references.bib
@inproceedings{kajiya1986rendering,
  author    = {Kajiya, James T.},
  title     = {The Rendering Equation},
  booktitle = {Proceedings of SIGGRAPH},
  pages     = {143--150},
  year      = {1986},
  doi       = {10.1145/15922.15902}
}

@inproceedings{veach1995optimally,
  author    = {Veach, Eric and Guibas, Leonidas J.},
  title     = {Optimally Combining Sampling Techniques for Monte Carlo Rendering},
  booktitle = {Proceedings of SIGGRAPH},
  pages     = {419--428},
  year      = {1995},
  doi       = {10.1145/218380.218498}
}

@article{jakob2012manifold,
  author  = {Jakob, Wenzel and Marschner, Steve},
  title   = {Manifold Exploration: A Markov Chain Monte Carlo Technique for Rendering Scenes with Difficult Specular Transport},
  journal = {ACM Transactions on Graphics},
  volume  = {31},
  number  = {4},
  year    = {2012},
  doi     = {10.1145/2185520.2185554}
}

@article{zeltner2020specular,
  author  = {Zeltner, Tizian and Georgiev, Iliyan and Jakob, Wenzel},
  title   = {Specular Manifold Sampling for Rendering High-Frequency Caustics and Glints},
  journal = {ACM Transactions on Graphics},
  volume  = {39},
  number  = {4},
  year    = {2020}
}

@article{fan2023manifold_path_guiding,
  title   = {Manifold Path Guiding for Importance Sampling Specular Chains},
  author  = {Fan, Zhimin and Hong, Pengpei and Guo, Jie and Zou, Changqing and Guo, Yanwen and Yan, Ling-Qi},
  journal = {ACM Transactions on Graphics},
  volume  = {42},
  number  = {6},
  articleno = {1},
  numpages  = {15},
  year    = {2023},
  doi     = {10.1145/3618360}
}

@inproceedings{dong2023npm,
  author    = {Dong, Honghao and Wang, Guoping and Li, Sheng},
  title     = {Neural Parametric Mixtures for Path Guiding},
  booktitle = {ACM SIGGRAPH Conference Proceedings},
  articleno = {29},
  pages     = {1--10},
  year      = {2023},
  doi       = {10.1145/3588432.3591533}
}

@inproceedings{muller2017practical,
  author    = {M{\"u}ller, Thomas and Gross, Markus and Nov{\'a}k, Jan},
  title     = {Practical Path Guiding for Efficient Light-Transport Simulation},
  booktitle = {Computer Graphics Forum},
  volume    = {36},
  pages     = {91--100},
  year      = {2017}
}

@article{wood1994vmf,
  author  = {Wood, Andrew T. A.},
  title   = {Simulation of the von Mises Fisher Distribution},
  journal = {Communications in Statistics -- Simulation and Computation},
  volume  = {23},
  number  = {1},
  pages   = {157--164},
  year    = {1994},
  doi     = {10.1080/03610919408813161}
}

@book{mardia2000directional,
  title     = {Directional Statistics},
  author    = {Mardia, Kanti V. and Jupp, Peter E.},
  publisher = {Wiley},
  year      = {2000}
}

@article{vorba2014online,
  author  = {Vorba, Ji{\v r}{\'i} and Karl{\'i}k, Ond{\v r}ej and {\v S}{\'i}k, Martin and Ritschel, Tobias and K{\v r}iv{\'a}nek, Jaroslav},
  title   = {On-line Learning of Parametric Mixture Models for Light Transport Simulation},
  journal = {ACM Transactions on Graphics},
  volume  = {33},
  number  = {4},
  articleno = {101},
  year    = {2014},
  doi     = {10.1145/2601097.2601203}
}

@article{herholz2016product,
  author  = {Herholz, Sebastian and Elek, Oskar and Vorba, Ji{\v r}{\'i} and Lensch, Hendrik P. A. and K{\v r}iv{\'a}nek, Jaroslav},
  title   = {Product Importance Sampling for Light Transport Path Guiding},
  journal = {Computer Graphics Forum},
  volume  = {35},
  number  = {4},
  year    = {2016}
}

@article{muller2019neural,
  author  = {M{\"u}ller, Thomas and McWilliams, Brian and Rousselle, Fabrice and Gross, Markus and Nov{\'a}k, Jan},
  title   = {Neural Importance Sampling},
  journal = {ACM Transactions on Graphics},
  volume  = {38},
  number  = {5},
  year    = {2019}
}

@article{ruppert2020robust,
  author  = {Ruppert, Lukas and Herholz, Sebastian and Lensch, Hendrik P. A.},
  title   = {Robust Fitting of Parallax-Aware Mixtures for Path Guiding},
  journal = {ACM Transactions on Graphics},
  volume  = {39},
  number  = {4},
  articleno = {147},
  year    = {2020}
}

@article{rath2020variance,
  author  = {Rath, Alexander and Grittmann, Pascal and Herholz, Sebastian and V{\'e}voda, Petr and Slusallek, Philipp and K{\v r}iv{\'a}nek, Jaroslav},
  title   = {Variance-Aware Path Guiding},
  journal = {ACM Transactions on Graphics},
  volume  = {39},
  number  = {4},
  articleno = {151},
  year    = {2020}
}

@article{hanika2015mnee,
  author  = {Hanika, Johannes and Droske, Marc and Fascione, Luca},
  title   = {Manifold Next Event Estimation},
  journal = {Computer Graphics Forum},
  volume  = {34},
  number  = {4},
  year    = {2015},
  doi     = {10.1111/cgf.12681}
}

@article{zheng2019primary,
  author  = {Zheng, Quan and Zwicker, Matthias},
  title   = {Learning to Importance Sample in Primary Sample Space},
  journal = {Computer Graphics Forum},
  volume  = {38},
  number  = {2},
  pages   = {169--179},
  year    = {2019},
  doi     = {10.1111/cgf.13628}
}

@article{bako2019offline,
  author  = {Bako, Steve and Meyer, Mark and DeRose, Tony and Sen, Pradeep},
  title   = {Offline Deep Importance Sampling for Monte Carlo Path Tracing},
  journal = {Computer Graphics Forum},
  volume  = {38},
  number  = {7},
  year    = {2019}
}

@article{fan2024specular,
  author  = {Fan, Zhimin and Guo, Jie and Wang, Yiming and Xiao, Tianyu and Zhang, Hao and Zhou, Chenxi and Chen, Zhenyu and Hong, Pengpei and Guo, Yanwen and Yan, Ling-Qi},
  title   = {Specular Polynomials},
  journal = {ACM Transactions on Graphics},
  volume  = {43},
  number  = {4},
  pages   = {1--13},
  year    = {2024},
  doi     = {10.1145/3658132}
}

@article{xu2026polynomial,
  title={Polynomial convergence of riemannian diffusion models},
  author={Xu, Xingyu and Zhang, Ziyi and Nakahira, Yorie and Qu, Guannan and Chi, Yuejie},
  journal={arXiv preprint arXiv:2601.02499},
  year={2026}
}

@article{xu2024provably,
  title={Provably robust score-based diffusion posterior sampling for plug-and-play image reconstruction},
  author={Xu, Xingyu and Chi, Yuejie},
  journal={Advances in Neural Information Processing Systems},
  volume={37},
  pages={36148--36184},
  year={2024}
}

@inproceedings{banerjee2025adaptive,
  title={Adaptive multimodal protein plug-and-play with diffusion-based priors},
  author={Banerjee, Amartya and Xu, Xingyu and Moosm{\"u}ller, Caroline and Lee, Harlin},
  booktitle={2025 IEEE 10th International Workshop on Computational Advances in Multi-Sensor Adaptive Processing (CAMSAP)},
  pages={226--230},
  year={2025},
  organization={IEEE}
}

@inproceedings{hoshinotraining,
  title={Training-Free Guided Diffusion for Planning: A Unified Framework via Doob’s h-Transform with Safety Guarantees},
  author={Hoshino, Kenta and Aluru, Yashaswi Shashank and Deng, Xiyu and Nakahira, Yorie},
  year={2026},
  booktitle={Forty-third International Conference on Machine Learning}
}

@misc{zhang2026hilbertgen3dhilbertmultifractalconditioningtopologyaware3d,
      title={HilbertGen-3D:Hilbert--Multifractal Conditioning for Topology-Aware 3D Generation}, 
      author={Ziyi Zhang and Shixuan Li and Yorie Nakahira and Paul Bogdan and Guannan Qu},
      year={2026},
      eprint={2610.22720},
      archivePrefix={arXiv},
      primaryClass={cs.CV},
      url={https://arxiv.org/abs/2610.22720}, 
}

@article{yin2026recursive,
  author  = {Yin, Shangjian and Zhao, Zehao and Asadi, Kavosh and Liu, Rui and Lu, Yuchen and Mei, Shike and Cui, Hang and Simon, Luke and Shi, Zhouxing and Firooz, Hamed},
  title   = {Recursive Self-Improvement via On-Policy Distillation for Reasoning},
  journal = {arXiv preprint arXiv:2609.30652},
  year    = {2026},
  doi     = {10.48550/arXiv.2609.30652},
  url     = {https://arxiv.org/abs/2609.30652}
}

@article{yueh2026bayesian,
  author  = {Yueh, Chang-Wei and Zhao, Andy and Nayyar, Ashutosh and Jain, Rahul},
  title   = {{Bayesian} Learning in Episodic Zero-Sum Games},
  journal = {arXiv preprint arXiv:2603.20604},
  year    = {2026},
  doi     = {10.48550/arXiv.2603.20604},
  url     = {https://arxiv.org/abs/2603.20604}
}

@article{zhao2025efficient,
  author  = {Zhao, Zehao and Shukla, Apurv and Jain, Rahul and Subramanian, Vijay},
  title   = {Efficient Decentralized Learning of Generalized Quantal Response Equilibrium},
  journal = {arXiv preprint arXiv:2507.09928},
  year    = {2025},
  doi     = {10.48550/arXiv.2507.09928},
  url     = {https://arxiv.org/abs/2507.09928}
}
